\documentclass[twoside]{article}
\usepackage[accepted]{aistats2012}

\usepackage[numbers]{natbib}
\usepackage{color}
\usepackage{latexsym}
\usepackage{amsmath,amssymb,amsthm,amsfonts}
\usepackage{algorithm,algorithmic}
\usepackage{graphicx,subfigure}

\newtheorem{lemma}{Lemma}

\newtheorem{theorem}{Theorem}

\newcommand{\tracer}[2]{\ensuremath{\langle \!\langle {#1}, \; {#2}
\rangle \!\rangle}}
\newcommand\myparagraph[1]{\vskip0.05in\noindent {\it #1.}}
\newcommand{\comment}[1]{}

\def\pdim{p}
\newcommand{\obsind}{\ensuremath{k}}
\newcommand{\numobs}{\ensuremath{n}}
\newcommand{\spindex}{\ensuremath{s}}

\newcommand{\SigHat}{\ensuremath{\widehat{\Sigma}}}
\newcommand{\ThetaStar}{\ensuremath{\Theta^*}}
\newcommand{\ThetaHat}{\ensuremath{\widehat{\Theta}}}

\newcommand{\regpar}{\ensuremath{\lambda_\numobs}}

\newcommand{\ellreg}[1]{\ensuremath{\|#1\|_{1,\operatorname{off}}}}
\newcommand{\CovMat}{\ensuremath{\Sigma}}
\newcommand{\CovMatStar}{\ensuremath{\CovMat^*}}
\newcommand{\graph}{\ensuremath{G}}

\newcommand{\vertex}{\ensuremath{V}}
\newcommand{\edge}{\ensuremath{E}}
\newcommand{\nodedeg}{\ensuremath{d}}
\newcommand{\degmax}{\ensuremath{\nodedeg}}

\newcommand{\defn}{\ensuremath{:  =}}

\newcommand{\Symconeplpl}[1]{\ensuremath{\mathcal{S}^{#1}_{+\!+}}}
\def\Loss{\mathcal{L}}
\def\supp{S}
\def\suppHat{\widehat{\supp}}
\def\suppStar{\supp^*}

\def\EpsilonStop{\epsilon_{\mathcal{S}}}

\def\svert{\ensuremath{r}}

\def\ThetaCond{\Gamma}
\def\real{\mathbb{R}}

\def\ThetaHatk{\widehat{\Theta}^{(k)}}
\def\SuppHatk{\widehat{S}^{(k)}}
\def\ThetaHatkm{\widehat{\Theta}^{(k-1)}}
\def\SuppHatkm{\widehat{S}^{(k-1)}}

\def\ThetaCondHat{\widehat{\ThetaCond}_\svert}
\def\ThetaCondHatk{\widehat{\ThetaCond}^{(k)}_\svert}
\def\ThetaCondHatkm{\widehat{\ThetaCond}^{(k-1)}_\svert}
\def\ThetaCondHatkmt{\widehat{\ThetaCond}^{(k-1)}_{\svert t}}
\def\ThetaCondHatkmtstar{\widehat{\ThetaCond}^{(k-1)}_{\svert t^*}}

\def\G{\mathcal{G}}
\def\V{\mathcal{V}}
\def\E{\mathcal{E}}

\begin{document}
\twocolumn[

\aistatstitle{High-dimensional Sparse Inverse Covariance Estimation using Greedy Methods}

\aistatsauthor{ Christopher C. Johnson \And Ali Jalali \And Pradeep Ravikumar }
\aistatsaddress{ CS, UT Austin\\ cjohnson@cs.utexas.edu \And ECE, UT Austin\\ alij@mail.utexas.edu \And CS, UT Austin\\ pradeepr@cs.utexas.edu } 
]

\begin{abstract}
In this paper we consider the task of estimating the non-zero pattern of the sparse inverse covariance matrix of a zero-mean Gaussian random vector from a set of iid samples. Note that this is also equivalent to recovering the underlying graph structure of a sparse Gaussian Markov Random Field (GMRF).  We present two novel greedy approaches to solving this problem.  The first estimates the non-zero covariates of the overall inverse covariance matrix using a series of global forward and backward greedy steps.  The second estimates the neighborhood of each node in the graph separately, again using greedy forward and backward steps, and combines the intermediate neighborhoods to form an overall estimate. The principal contribution of this paper is a rigorous analysis of the sparsistency, or consistency in recovering the sparsity pattern of the inverse covariance matrix. Surprisingly, we show that both the local and global greedy methods learn the full structure of the model with high probability given just $O(d\log(p))$ samples, which is a \emph{significant} improvement over state of the art $\ell_1$-regularized Gaussian MLE (Graphical Lasso) that requires $O(d^2\log(p))$ samples. Moreover, the restricted eigenvalue and smoothness conditions imposed by our greedy methods are much weaker than the strong irrepresentable conditions required by the $\ell_1$-regularization based methods. We corroborate our results with extensive simulations and examples, comparing our local and global greedy methods to the $\ell_1$-regularized Gaussian MLE as well as the Neighborhood Greedy method to that of nodewise $\ell_1$-regularized linear regression (Neighborhood Lasso).
\end{abstract}

\section{Introduction}
\myparagraph{High-dimensional Covariance Estimation}
Increasingly, modern statistical problems across varied fields of science and engineering involve a large number of variables. Estimation of such high-dimensional models has been the focus of considerable recent research, and it is now well understood that consistent estimation is possible when some low-dimensional structure is imposed on the model space. In this paper, we consider the specific high-dimensional problem of recovering the covariance matrix of a zero-mean Gaussian random vector,
under the low-dimensional structural constraint of \emph{sparsity} of the inverse covariance, or concentration matrix. When the random vector is multivariate Gaussian, the set of non-zero entries in the concentration matrix correspond to the set of edges in an associated Gaussian Markov random field (GMRF). In this setting, imposing sparsity on the entries of the concentration matrix can be interpreted as requiring that the graph underlying the GMRF have relatively few edges.

\myparagraph{State of the art: $\ell_1$ regularized Gaussian MLE}
For this task of sparse GMRF estimation, a line of recent papers~\citep{AspreBanG2008,FriedHasTib2007,YuanLin2007} have proposed an estimator that minimizes the Gaussian negative log-likelihood regularized by the $\ell_1$ norm of the entries (or the off-diagonal entries) of the concentration matrix. The resulting optimization problem is a log-determinant program, which can be solved in polynomial time with interior point methods~\citep{Boyd02}, or by co-ordinate descent algorithms~\citep{AspreBanG2008,FriedHasTib2007}. \citet{Rothman2007,RWRY11} have also shown strong statistical guarantees for this estimator: both in $\ell_2$ operator norm error bounds, and recovery of the underlying graph structure.

\myparagraph{Recent resurgence of greedy methods}
A related line of recent work on learning sparse models has focused on ``stagewise'' greedy algorithms. These perform simple forward steps (adding parameters greedily), and possibly also backward steps (removing parameters greedily), and yet provide strong statistical guarantees for the estimate after a finite number of greedy steps. Indeed, such greedy algorithms have appeared in various guises in multiple communities: in machine learning as boosting~\cite{Friedman00}, in function approximation~\cite{Temlyakov08}, and in signal processing as basis pursuit~\cite{Chen98}. In the context of statistical model estimation, \citet{Zhang09} analyzed the forward greedy algorithm for the case of sparse linear regression; and showed that the forward greedy algorithm is sparsistent (consistent for model selection recovery) under the same ``irrepresentable'' condition as that required for ``sparsistency'' of the Lasso. \citet{Zhang08} analyzes a more general greedy algorithm for sparse linear regression that performs forward and backward steps, and showed that it is sparsistent under a weaker restricted eigenvalue condition. \citet{JJR11} extend the sparsistency analysis of \cite{Zhang08} to general non-linear models, and again show that strong sparsistency guarantees hold for these algorithms.

\myparagraph{Our Approaches}
Motivated by these recent results, we apply the forward-backward greedy algorithm studied in \cite{Zhang08,JJR11} to the task of learning the graph structure of a
Gaussian Markov random field given iid samples. We propose two algorithms: one that applies the greedy algorithm to the overall Gaussian log-likelihood loss,
and the other that is based on greedy neighborhood estimation. For this second method, we follow \cite{MeinsBuhl2006,RWL10}, and estimate the neighborhood of each node by applying the greedy algorithm to the local node conditional log-likelihood loss (which reduces to the least squares loss), and then show that each neighborhood is recovered with very high probability, so that by an elementary union bound, the entire graph structure is recovered with high probability. A principal contribution of this paper is a rigorous analysis of these algorithms, where we report sufficient conditions for recovery of the underlying graph structure. We also corroborate our analysis with extensive simulations.

Our analysis shows that for a Gaussian random vector $X = (X_1, X_2,\ldots,X_p)$ with $p$ variables, both the global and local greedy algorithms only require $n = O(d \log(p))$ samples
for sparsistent graph recovery. Note that this is a significant improvement over the $\ell_1$ regularized Gaussian MLE~\cite{YuanLin2007} which has been shown to require $O(d^2\log(p))$ samples~\cite{RWRY11}. Moreover, we show that the local and global greedy algorithms require a very weak restricted eigenvalue and restricted smoothness condition on the true inverse covariance matrix (with the local greedy imposing a marginally weaker condition that the global greedy algorithm). This is in contrast to the $\ell_1$ regularized Gaussian MLE which imposes a very stringent edge-based irrepresentable condition~\cite{RWRY11}. In Section~\ref{SecCompare}, we explicitly compare these different conditions imposed by the various methods for some simple GMRFs, and quantitatively show that the conditions imposed by the local and global greedy methods require much weaker conditions on the covariance entries. Thus, both theoretically and via simulations, we show that the set of methods proposed in the paper are the \emph{state of the art} in recovering the graph structure of a GMRF from iid samples: both in the number of samples required, and the weakness of the sufficient conditions imposed upon the model.

\section{Problem Setup}
\subsection{Gaussian graphical models}

Let $X = (X_1, X_2, \ldots, X_\pdim)$ be a zero-mean Gaussian random vector. Its density is parameterized by its inverse
covariance or \emph{concentration matrix} $\ThetaStar = (\CovMatStar)^{-1} \succ 0$, and can be written as
\begin{eqnarray}
\label{EqnDefnGaussMRF}
f(x_1, \ldots, x_\pdim; \ThetaStar) =  \frac{\exp \big\{ -\frac{1}{2} x^T \ThetaStar x \big \}}{\sqrt{(2 \pi)^\pdim
 \det(\ThetaStar)^{-1}}}.
\end{eqnarray}

We can associate an undirected graph structure $\graph = (\vertex, \edge)$ with this distribution, with the vertex set $\vertex = \{1, 2, \ldots, \pdim \}$ corresponding 
to the variables $(X_1, \ldots, X_\pdim)$, and with edge set such that $(i,j) \notin \edge$ if $\ThetaStar_{ij} = 0$. 

We are interested in the problem of recovering this underlying graph structure, which corresponds to determining
which off-diagonal entries of $\ThetaStar$ are non-zero---that is, the set
\begin{eqnarray}
\label{EqnDefnEdgeSet}
\edge(\ThetaStar) & \defn & \{i, j \in \vertex \mid \, i \neq j,
\ThetaStar_{ij} \neq 0 \}.
\end{eqnarray}
Given $\numobs$ samples, we define the \emph{sample covariance matrix}
\begin{eqnarray}
\label{EqnDefnSamCov}
\SigHat^\numobs & \defn & \frac{1}{\numobs} \sum_{\obsind=1}^\numobs X^{(\obsind)} (X^{(\obsind)})^T.
\end{eqnarray}
In the sequel, we occasionally drop the superscript $\numobs$, and simply write $\SigHat$ for the sample covariance.

With a slight abuse of notation, we define the \emph{sparsity index} $\spindex \defn |\edge(\ThetaStar)|$ as the total number of non-zero
elements in off-diagonal positions of $\ThetaStar$; equivalently, this corresponds to twice the number of edges in the case of a Gaussian
graphical model. We also define the \emph{maximum degree or row cardinality}
\begin{eqnarray}
\label{EqnDefnDegmax}
\degmax & \defn & \max_{i = 1, \ldots, \pdim } \biggr|\big \{ j \in
\vertex \, \mid \, \ThetaStar_{ij} \neq 0 \big\} \biggr|,
\end{eqnarray}
corresponding to the maximum number of non-zeros in any row of $\ThetaStar$; this corresponds to the maximum degree in the graph of
the underlying Gaussian graphical model.  Note that we have included the diagonal entry $\ThetaStar_{ii}$ in the degree count,
corresponding to a self-loop at each vertex.

\subsection{State of the art: $\ell_1$ regularization}

Define the \emph{off-diagonal $\ell_1$ regularizer}
\begin{eqnarray}\label{EqnOffDiagEll1}
\ellreg{\Theta} & \defn & \sum_{i \neq j} |\Theta_{ij}|,
\end{eqnarray}
where the sum ranges over all $i, j = 1, \ldots, \pdim$ with $i \neq
j$.  Given some regularization constant $\regpar > 0$, we consider
estimating $\ThetaStar$ by solving the following \emph{$\ell_1$-regularized log-determinant program}:
\begin{equation}
\label{EqnGaussMLE}
\ThetaHat\defn  \arg\min_{\Theta \in \Symconeplpl{\pdim}} \big \{
\tracer{\Theta}{\SigHat^\numobs} - \log \det(\Theta) + \regpar
\ellreg{\Theta} \big \},
\end{equation}
which returns a symmetric positive definite matrix $\ThetaHat$.

Note that this corresponds to the $\ell_1$ regularized Gaussian MLE when the 
underlying distribution is Gaussian.

\subsection{Forward Backward Greedy}

\cite{Zhang08,JJR11} consider a simple forward-backward greedy algorithm for model estimation that begins with an empty set of active variables and gradually adds (and removes) variables to the active set. This algorithm has two basic steps: the forward step and the backward step. In the forward step, the algorithm finds the \emph{best} next candidate and adds it to the active set as long as it improves the loss function at least by $\EpsilonStop$, otherwise the stopping criterion is met and the algorithm terminates. Then, in the backward step, the algorithm checks the \emph{influence} of all variables in the presence of the newly added variable. If one or more of the previously added variables do not contribute at least $\nu\EpsilonStop$ to the loss function, then the algorithm removes them from the active set. This procedure ensures that at each round, the loss function is improved by at least $(1-\nu)\EpsilonStop$ and hence it terminates within a finite number of steps.

In the sequel, we will apply this greedy methodology to Gaussian graphical models, to obtain two methods: (a) Greedy Gaussian MLE, which applies the greedy algorithm
to the Gaussian negative log-likelihood loss, and (b) Greedy Neighborhood Estimation, which applies the greedy algorithm to the local node-conditional negative log-likelihood loss.

\section{Greedy Gaussian MLE}
In Algorithm \ref{Alg:globalgreedyalg}, we describe the greedy algorithm of \cite{Zhang08,JJR11} as applied to the Gaussian log-likelihood loss,
\begin{align*} 
	\Loss(\Theta) := \tracer{\Theta}{\SigHat^\numobs} - \log \det(\Theta).
\end{align*}


\begin{algorithm}[t]
\caption{ Global greedy forward-backward algorithm for Gaussian covariance estimation}
\label{Alg:globalgreedyalg}
\begin{algorithmic}
\STATE  {\bf Input}: $ \SigHat^\numobs$, Stopping Threshold $\epsilon_S$, Backward Step Factor $\nu \in (0,1)$
\STATE  {\bf Output}: Inverse Covariance Estimation $\ThetaHat$
\STATE
\STATE Initialize $\ThetaHat^{(0)}\leftarrow \mathbb{I}$, $\widehat{S}^{(0)}\leftarrow \emptyset$, and $k \leftarrow 1$ 
\STATE
\WHILE[\textit{Forward Step}]{true}
	\STATE {\small $\displaystyle((i_*,j_*),\alpha_*)\longleftarrow\arg\!\!\!\! \min_{(i,j)\in\left(\SuppHatkm\right)^c\,;\,\alpha}\!\!\!\!\Loss\left(\ThetaHatkm\!\!\!+\!\alpha (e_{ij}\!\!+\!e_{ji})\right)$}
	\STATE {\small $\SuppHatk \longleftarrow \SuppHatkm \cup \{(i_*,j_*)\}$}\\
	\STATE {\small $\delta_f^{(k)} \longleftarrow \Loss(\ThetaHatkm) - \Loss\left(\ThetaHatkm+\alpha_*(e_{i_*j_*}\!\!+\!e_{j_*i_*})\right)$}
	\IF {\small $\delta_f^{(k)}\leq\epsilon_S$}
	\STATE {\small \bf break}
	\ENDIF
	\STATE
	\STATE {\small $\displaystyle \ThetaHatk \longleftarrow 
			\arg\min_{\,\Theta} \,\Loss\big(\Theta_{\SuppHatk}\big)$}\\
	\STATE {\small $k \longleftarrow k+1$}\\
	\STATE
	\WHILE[\textit{Backward Step}]{true}
	\STATE {\small $\displaystyle (i^*,j^*) \longleftarrow \arg\min_{j \in \SuppHatkm}\Loss\left(\ThetaHatkm-\ThetaHatkm_{ij}(e_{ij}\!+\!e_{ji})\right)$}\\
	\IF {\small $\Loss\big(\ThetaHatkm\! -\! \ThetaHatkm_{i^*j^*}(e_{i^*j^*}\!+e_{j^*i^*})\big)\! -\! \Loss\big(\ThetaHatkm\big)\!\! >\!\! \nu\delta_f^{(k)}$}
	\STATE {\small \bf break}
	\ENDIF
	\STATE
	\STATE {\small $\SuppHatkm\longleftarrow\SuppHatk-\{(i^*,j^*)\}$}\\
	\STATE {\small $\displaystyle \ThetaHatkm \longleftarrow 
			\arg\min_{\,\theta} \,\Loss\big(\Theta_{\SuppHatkm}\big)$}\\
	\STATE {\small $k\longleftarrow k-1$}\\
	\ENDWHILE
	\STATE
\ENDWHILE
\STATE
\end{algorithmic}
\end{algorithm}

\vskip0.2in
\noindent {\bf Assumption:}\\
\noindent Let $\rho\geq 1$ be a constant and $\Delta\in\real^{p\times p}$ be a symmetric matrix that is \emph{sparse} with at most $\eta d$ non-zero entries per row (and column) for some $\eta\geq 2+4\rho^2(\sqrt{(\rho^2-\rho)/d}+\sqrt{2})^2$. We require that population covariance matrix $\Sigma^*=\mathbb{E}\left[XX^T\right]$ satisfy the restricted eigenvalue property, i.e., for some positive constants $C_{\min}$, we have
\begin{equation}\label{AssumpCov}
C_{\min}\left\|\Delta\right\|_F\leq\tracer{\Sigma^*}{\Delta}\leq\rho C_{\min}\left\|\Delta\right\|_F,
\nonumber
\end{equation}
where, $\|\cdot\|_F$ denotes the Frobenius norm.

\begin{lemma}
Suppose $\Sigma^*$ satisfies the assumption in \ref{AssumpCov}. Then, with probability at least $1 - c_1 \exp(- c_2 n)$ for arbitrary small constant $\alpha>0$, we have
that for any symmetric matrix $\Delta$ with $\eta d$ non-zero entries per row (and column),
\begin{equation}
(1-\alpha)C_{\min}\left\|\Delta\right\|_F\leq\tracer{\widehat{\Sigma}^n}{\Delta}\leq (1+\alpha)\rho C_{\min}\left\|\Delta\right\|_F,
\nonumber
\end{equation}
provided that $n\geq K\,d\,\log(p)$ for some positive constant $K$, $c_1$ and $c_2$.
\end{lemma}

\begin{proof}
The proof follows from Lemma 9 (Appendix K) in \cite{Wainwright09}.\\
\end{proof}

\noindent Using Taylor series expansion, we can write
\begin{equation}
\begin{aligned}
	\Loss(\Theta+\Delta) &= \Loss(\Theta) + \tracer{\Delta}{\SigHat^\numobs} - \tracer{\Theta^{-1}}{\Delta}\\ & \qquad\quad+ \underbrace{\sum_{i=2}^\infty \frac{(-1)^i}{i}\tracer{\;(\Theta^{-1}\Delta)^{i-1}\,\Theta^{-1}}{\Delta}}_{R_\Delta}.
\end{aligned}
\nonumber
\end{equation}
In order to establish the restricted strong convexity/smoothness required by \cite{JJR11}, we need to lower/upper bound $R_\Delta$. Notice that in the proof of \cite{JJR11}, the required $\Delta$ is the difference between the target variable $\Theta^*$ and the $k^{th}$ step estimation $\widehat{\Theta}^{(k)}$. Since the algorithm is guaranteed to converge, $\Delta = \Theta^*-\widehat{\Theta}^{(k)}$ is always bounded. Thus, without loss of generality, we assume that $\|\Delta\|_F\leq 1$. Notice that we can scale $\|\Delta\|_F$ and similar type of result holds. The next lemma provides the required upper/lower bound.

\begin{lemma}
Suppose $\Sigma^*$ satisfies the assumption in \ref{AssumpCov}. Then with probability at least $1 - c_1 \exp(- c_2 n)$, we have
that for any symmetric matrix $\Delta$ with $\eta d$ non-zero entries per row (and column), and with $\|\Delta\|_F\leq 1$, 
\begin{equation}
\begin{aligned}
\frac{1}{4}C_{\min}^2\|\Delta\|_F^2 \leq R_\Delta \leq \frac{1}{2}\rho^2 C_{\min}^2 \|\Delta\|_F^2.
\end{aligned}
\nonumber
\end{equation}
\end{lemma}

\begin{proof}
Denote $\gamma=\tracer{\Theta^{-1}}{\Delta}$. We have
\begin{equation}
R_\Delta = \sum_{i=2}^\infty \frac{(-1)^i}{i}\gamma^i = \gamma - \log(1+\gamma).
\nonumber
\end{equation}
Under our assumption, $C_{\min}\|\Delta\|_F\leq\gamma\leq\rho C_{\min}\|\Delta\|_F$ and the function $\gamma - \log(1+\gamma)$ is an increasing function in $\gamma$. Moreover, for the range of $\gamma$, we have $\gamma - \log(1+\gamma)\geq\frac{1}{4}\gamma^2$ because they both vanish at zero and the derivative of LHS is larger than the derivative of LHS. Hence, we have
\begin{equation}
\begin{aligned}
\frac{1}{4}C_{\min}^2\|\Delta\|_F^2&\leq C_{\min}\|\Delta\|_F - \log(1+C_{\min}\|\Delta\|_F)\\ &\leq \gamma - \log(1+\gamma) = R_\Delta\\
&\leq \rho C_{\min}\|\Delta\|_F - \log(1+\rho C_{\min}\|\Delta\|_F)\\
&\leq \frac{1}{2}\rho^2 C_{\min}^2 \|\Delta\|_F^2. 
\end{aligned}
\nonumber
\end{equation}
The last inequality follows from $\gamma - \log(1+\gamma)\leq \frac{1}{2}\gamma^2$ (since they are equal at zero and the derivative of RHS is always larger above zero). Hence, the result follows.\\
\end{proof}

Let $\nabla^{(n)}:=\|\SigHat^\numobs - (\Theta^{*})^{\!-1}\|_\infty$. By first order condition on the optimality of $\Theta^*$, it is clear that $\lim_{n\rightarrow\infty}\nabla^{(n)}=0$. The following lemma provides an upper bound on $\nabla^{(n)}$.

\begin{lemma}
Given the sample complexity $n\geq K\,\log(p)$ for some constant $K$, we have
\begin{equation}
\nabla^{(n)}\leq c\sqrt{\frac{\log(p)}{n}},
\nonumber
\end{equation}
with probability at least $1 - c_1 \exp(- c_2 n)$ for some positive constants $c$, $c_1$ and $c_2$.
\end{lemma}

\begin{proof}
The proof follows from Lemma 1 in \cite{RWRY11}.\\
\end{proof}

\noindent This entails that the restricted strong convexity and smoothness (i.e., the required assumptions of the general result in \cite{JJR11}) are satisfied. Now, we can specialize the results in \cite{JJR11} to obtain the following theorem:

\begin{theorem}[Global Greedy Sparsistency] \label{ThmGlobalGreedy}
Under the assumption above, suppose we run Algorithm \ref{Alg:globalgreedyalg} with stopping threshold $\EpsilonStop \ge (2c\eta/\rho^2)d\log(p)/n$, where, $d$ is the maximum node degree in the graphical model, and the true parameters $\Theta^*$ satisfy $\min_{t\in\mathcal{S}^*}|\Theta^*|\geq\sqrt{8\EpsilonStop/\rho^2}$, and further that number of samples scales as $$n > K\,d\,\log(p)$$ for some constant $K$. Then, with probability at least $1 - c_1 \exp(- c_2 n)$, we have 
\begin{itemize}
	\item[(a)] {\bf No False Exclusions:} $E^* - \widehat{E} = \emptyset.$
	\item[(b)] {\bf No False Inclusions:} $\widehat{E} - E^* = \emptyset.$
\end{itemize}	
\end{theorem}

\section{Greedy Neighborhood Estimation}
Denote by $\mathcal{N}^*(\svert)$ the set of neighbors of a vertex $\svert \in V$, so that $\mathcal{N}^*(\svert) = \{t : (\svert,t) \in E^*\}$. Then the graphical model selection problem is equivalent to that of estimating the neighborhoods $\hat{\mathcal{N}}_n(\svert) \subset \vertex$, so that $\mathbb{P}[\hat{\mathcal{N}}_n(\svert) = \mathcal{N}^*(\svert); \forall \svert \in \vertex]\rightarrow 1$ as $n \rightarrow \infty$. 

For any pair of random variables $X_\svert$ and $X_t$, the parameter $\Theta_{\svert t}$ fully characterizes whether there is an edge between them, and can be estimated via its conditional likelihood. In particular, defining $\Theta_\svert \defn \{\Theta_{\svert t}\}_{t \neq \svert}$, our goal is to use the conditional likelihood of $X_\svert$ conditioned on $X_{\vertex\backslash\svert}$ to estimate the \emph{support} of $\Theta_\svert$ and hence its neighborhood $\mathcal{N}(\svert)$. This conditional distribution of $X_\svert$ conditioned on $X_{\vertex\backslash\svert}$ generated by \eqref{EqnDefnGaussMRF} is given by (considering $\Theta^{-1}=\Sigma$)

\vspace{-0.4cm}
\small\begin{equation}
	X_\svert | X_{\vertex\backslash\svert} \sim \mathcal{N}\left(-\Theta^{-1}_{\svert\,\backslash\svert}\Theta_{\backslash\svert\backslash\svert} \, X_{\vertex\backslash\svert},\;\Theta^{-1}_{\svert\svert}\!-\!\Theta^{-1}_{\svert\,\backslash\svert} \Theta_{\backslash\svert\backslash\svert}\Theta^{-1}_{\backslash\svert\,\svert}\right).
\nonumber
\end{equation}\normalsize
However, note that we do not need to estimate the variance of this conditional distribution in order to obtain the support of $\Theta_\svert=\Theta_{\svert\,\backslash\svert}$. In
particular, the solution to the following least squares loss
\begin{align*}
	\ThetaCond_\svert^* = \arg\min_{\ThetaCond_\svert} \mathbb{E}[(X_\svert - \sum_{t \neq \svert} \ThetaCond_{\svert t} X_{t})^2],
\end{align*}
would satisfy $\text{supp}(\ThetaCond^*_\svert) = \text{supp}(\Theta_\svert^*)$.

Given the $n$ samples $X^{(1)},\ldots,X^{(n)}$, we thus use the sample-based linear loss
\begin{align}
\label{eq:Loss-Fn}
\Loss(\ThetaCond_\svert) = \frac{1}{2\numobs} \sum_{i=1}^\numobs \left(X^{(i)}_\svert - \sum_{t \neq \svert} \ThetaCond_{\svert t} X^{(i)}_{t}\right)^2.
\end{align}
Adapting the greedy algorithm from the previous section to this linear loss at each node thus yields Algorithm \ref{Alg:nbdgreedyalg}.

\begin{algorithm}[t]
\caption{ Greedy forward-backward algorithm for marginal Gaussian covariance estimation}
\label{Alg:nbdgreedyalg}
\begin{algorithmic}
\STATE  {\bf Input}: Data Vectors $X^{(1)},\ldots,X^{(n)}$, Stopping Threshold $\epsilon_S$, Backward Step Factor $\nu \in (0,1)$
\STATE  {\bf Output}: Marginal Vector $\widehat{\ThetaCond}_\svert$ 
\STATE
	\STATE Initialize $\ThetaCondHat^{(0)}\leftarrow \mathbf{0}$, $\widehat{S}^{(0)}\leftarrow \emptyset$, and $k \leftarrow 1$
	\WHILE[\textit{Forward Step}]{true}
		\STATE {\small $\displaystyle(t_*,\alpha_*)\longleftarrow\arg\!\!\! \min_{t\in\left(\SuppHatkm\right)^c\,;\,\alpha}\!   
							\Loss\left(\ThetaCondHatkm\!+\!\alpha e_t\right)$}
		\STATE {\small $\SuppHatk \longleftarrow \SuppHatkm \cup \{t_*\}$}\\
		\STATE {\small $\delta_f^{(k)} \longleftarrow \Loss(\ThetaCondHatkm) - \Loss(\ThetaCondHatkm+\alpha_*e_{t_*})$}
		\IF {\small $\delta_f^{(k)}\leq\epsilon_S$}
		\STATE {\small \bf break}
		\ENDIF
		\STATE
		\STATE {\small $\displaystyle \ThetaCondHatk \longleftarrow 
				\arg\min_{\,\Gamma_\svert} \,\Loss\big((\Gamma_{\svert})_{\,\SuppHatk}\big)$}\\
		\STATE {\small $k \longleftarrow k+1$}\\
		\STATE
		\WHILE[\textit{Backward Step}]{true}
		\STATE {\small $\displaystyle t^* \longleftarrow \arg\min_{t \in \SuppHatkm}\Loss(\ThetaCondHatkm-\ThetaCondHatkmt e_t)$}\\
		\IF {\small $\Loss\big(\ThetaCondHatkm - \ThetaCondHatkmtstar e_{t^*}\big) - \Loss\big(\ThetaCondHatkm\big) > \nu\delta_f^{(k)}$}
		\STATE {\small \bf break}
		\ENDIF
		\STATE
		\STATE {\small $\SuppHatkm\longleftarrow\SuppHatk-\{t^*\}$}\\
		\STATE {\small $\displaystyle \ThetaCondHatkm \longleftarrow 
				\arg\min_{\,\Gamma_\svert} \,\Loss\big((\Gamma_\svert)_{\SuppHatkm}\big)$}\\
		\STATE {\small $k\longleftarrow k-1$}\\
		\ENDWHILE
	\ENDWHILE
\STATE
\end{algorithmic}
\end{algorithm}

\vskip0.2in

\noindent {\bf Assumption:}\\
\noindent Let $\rho\geq 1$ be a constant and $\Delta\in\real^{p-1}$ be an arbitrary $\eta d$-sparse vector, where, $\eta\geq 2+4\rho^2(\sqrt{(\rho^2-\rho)/d}+\sqrt{2})^2$. We require the marginal population Fisher information matrix $\Sigma_{\backslash\svert}^*=\mathbb{E}\left[X_{\backslash\svert}X_{\backslash\svert}^T\right]$ satisfy the restricted eigenvalue property, i.e., for some positive constants $C_{\min}$, we have
\begin{equation}
C_{\min}\|\Delta\|_F\leq\|\Sigma_{\backslash\svert}^*\Delta\|_F\leq\rho C_{\min}\|\Delta\|_F.
\nonumber
\end{equation}

\begin{lemma}
Under assumption above, and for some arbitrary small constant $\alpha>0$, the marginal sample Fisher information matrix $\widehat{\Sigma}^{n}_{\backslash\svert} = \frac{1}{n}\sum_{i=1}^n X_{\backslash\svert}^{(i)}X_{\backslash\svert}^{(i)T}$, with probability at least $1 - c_1 \exp(- c_2 n)$,  satisfies the 
condition that for any symmetric matrix $\Delta$ with $\eta d$ non-zero entries per row (and column),
\begin{equation}
(1-\alpha)C_{\min}\|\Delta\|_F\leq\|\widehat{\Sigma}^{n}_{\backslash\svert}\Delta\|_F\leq (1+\alpha)\rho C_{\min}\|\Delta\|_F,
\nonumber
\end{equation}
provided that $n\geq K\,d\,\log(p)$ for some positive constant $K$, $c_1$ and $c_2$.
\end{lemma}

\begin{proof}
The proof follows from Lemma 9 (Appendix K) in \cite{Wainwright09}.\\
\end{proof}

Let $\nabla_\svert^{(n)}:=\max_{t}\left|\frac{1}{\numobs} \sum_{i=1}^\numobs X^{(i)}_{t} \left(X^{(i)}_\svert - \sum_{t \neq \svert} \ThetaCond^*_{\svert t} X^{(i)}_{t}\right)\right|$. By first order condition on the optimality of $\ThetaCond^*_{\svert t}$, it is clear that $\lim_{n\rightarrow\infty}\nabla_\svert^{(n)}=0$. The following lemma provides an upper bound on $\nabla_\svert^{(n)}$.

\begin{lemma}
Given the sample complexity $n\geq K\,\log(p)$ for some constant $K$, we have
\begin{equation}
\nabla_\svert^{(n)}\leq c\sqrt{\frac{\log(p)}{n}},
\nonumber
\end{equation}
with probability at least $1 - c_1 \exp(- c_2 n)$ for some positive constants $c$, $c_1$ and $c_2$.
\end{lemma}

\begin{proof}
The proof follows from Lemma 5 in \cite{Wainwright09}.\\
\end{proof}

\noindent This entails that the restricted strong convexity and smoothness (i.e., the required assumptions of the general result in \cite{JJR11}) are satisfied with constants $C_{\min}$ and $\rho C_{\min}$, respectively; because, the third and higher order derivatives are zero. Now, we can then specialize the results in \cite{JJR11} to obtain the following theorem:

\begin{theorem}[Neighborhood Greedy Sparsistency] \label{ThmNbdGreedy}
Under the assumption above, suppose we run Algorithm \ref{Alg:nbdgreedyalg} with stopping threshold $\EpsilonStop \ge (8c\rho\eta/C_{\min})d\log(p)/n$, where, $d$ is the maximum node degree in the graphical model, and the true parameters $\ThetaCond^*_\svert$ satisfy $\min_{t\in\mathcal{N}^*(\svert)}|\ThetaCond^*_{\svert t}|\geq\sqrt{32\rho\EpsilonStop/C_{\min}}$, and further that number of samples scales as $$n > K\,d\,\log(p)$$ for some constant $K$. Then, with probability at least $1 - c_1 \exp(- c_2 n)$, we have 
\begin{itemize}
	\item[(a)] {\bf No False Exclusions:} $E^*_\svert - \widehat{E}_\svert = \emptyset.$
	\item[(b)] {\bf No False Inclusions:} $\widehat{E}_\svert - E^*_\svert = \emptyset.$
\end{itemize}	
\end{theorem}

\section{Comparisons to Related Methods}\label{SecCompare}
In this section, we compare our global and local greedy methods to the $\ell_1$-regularized Gaussian MLE, analyzed in \cite{RWRY11}, and to $\ell_1$-regularization (Lasso) based neighborhood selection, analyzed in \cite{MeinsBuhl2006,Wainwright09}.

\subsection{Sample Complexity}
Our greedy algorithm requires $\mathcal{O}(d\,\log(p))$ samples to recover the exact structure of the graph for both the global and local neighborhood based methods.
In contrast, the $\ell_1$-regularized Gaussian MLE~\cite{RWRY11} requires $\mathcal{O}(d^2\,\log(p))$ samples to guarantee structure recovery with high probability.
The linear neighborhood selection with $\ell_1$-regularization \cite{MeinsBuhl2006} requires $\mathcal{O}(d\,\log(p))$ samples to guarantee sparsistency, similar to our greedy
algorithms.

\subsection{Minimum Non-Zero Values}
The $\ell_1$-regularized Gaussian MLE imposes the model condition that the minimum non-zero entry of $\Sigma^{*\,-1}$ satisfy $\Sigma^{*\,-1}_{\min}= \mathcal{O}(1/d)$. Our greedy algorithms allow for a broader range of minimum non-zero values $\Sigma^{*\,-1}_{\min} = \mathcal{O}(1/\sqrt{d})$. The linear neighborhood selection with $\ell_1$-regularization again matches our greedy algorithms and only requires that $\Sigma^{*\,-1}_{\min} = \mathcal{O}(1/\sqrt{d})$.

\begin{figure}[t]
\subfigure[Star]{
   \includegraphics[width=0.19\textwidth] {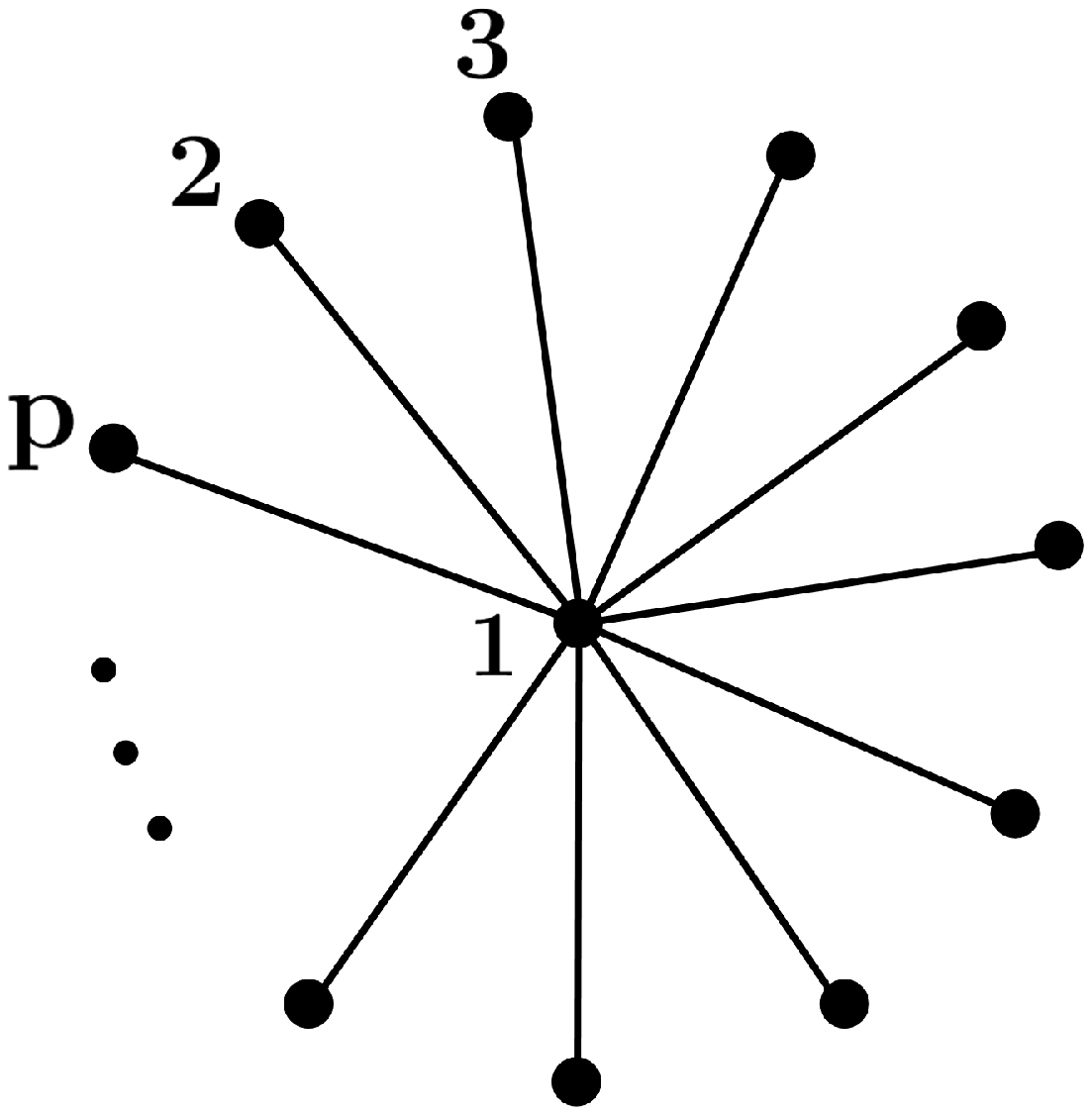}
   \label{fig1:star}
 }
 \subfigure[Chain]{
   \includegraphics[width=0.23\textwidth] {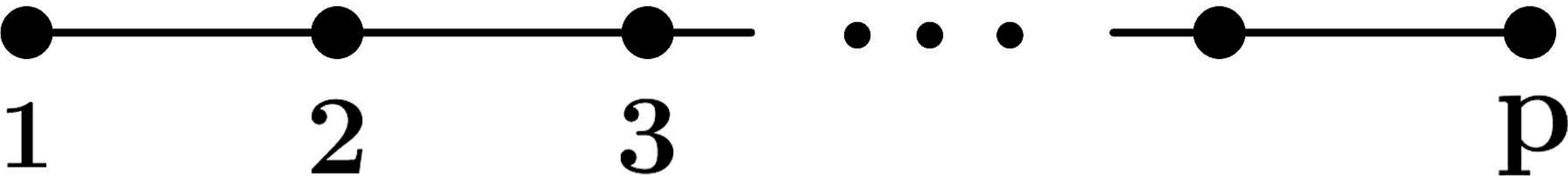}
   \label{fig1:line}
 }
  \subfigure[Grid]{
   \includegraphics[width=0.21\textwidth] {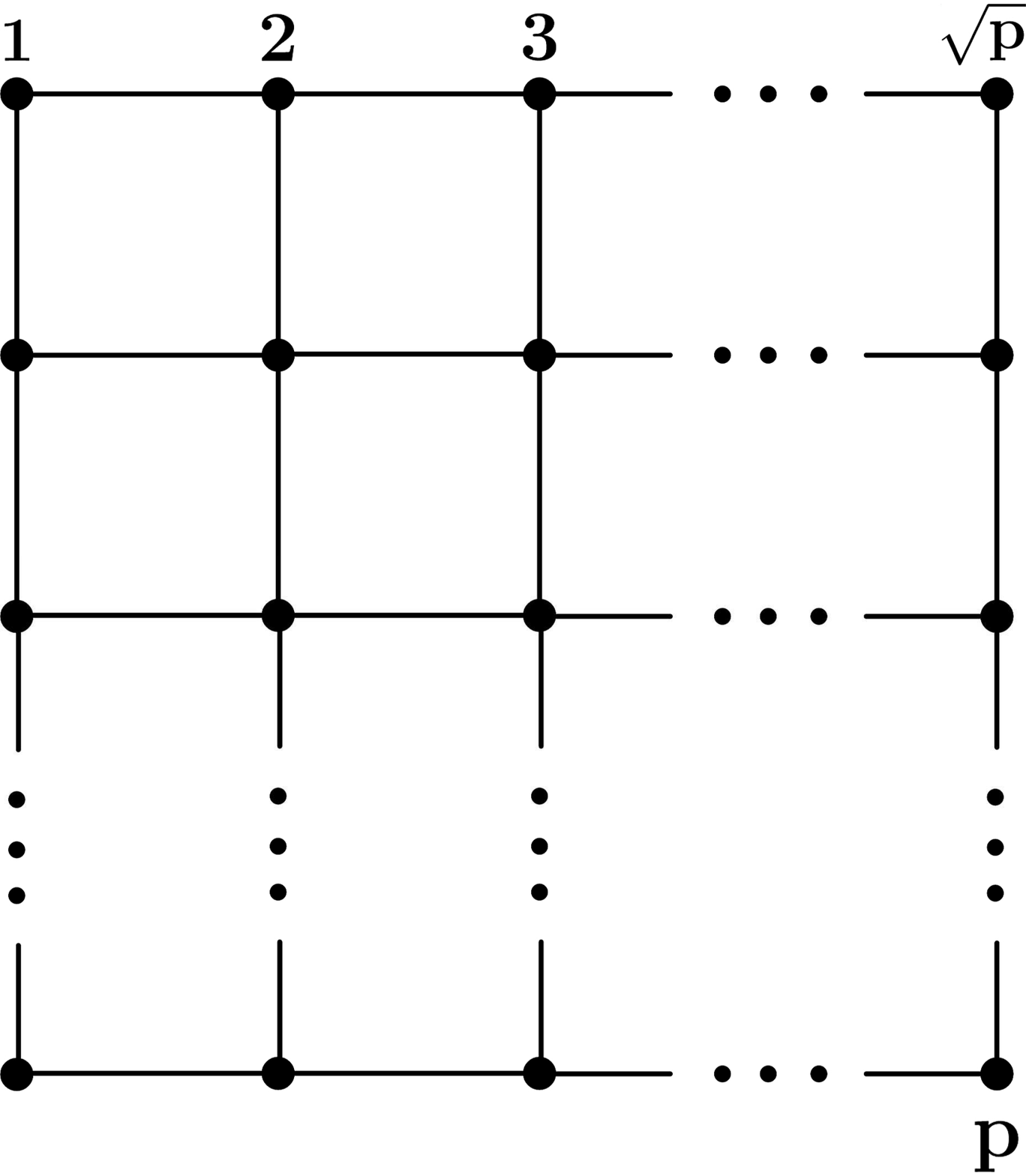}
   \label{fig1:grid}
 } $\qquad\quad$
 \subfigure[Diamond]{
  \includegraphics[width=0.13\textwidth] {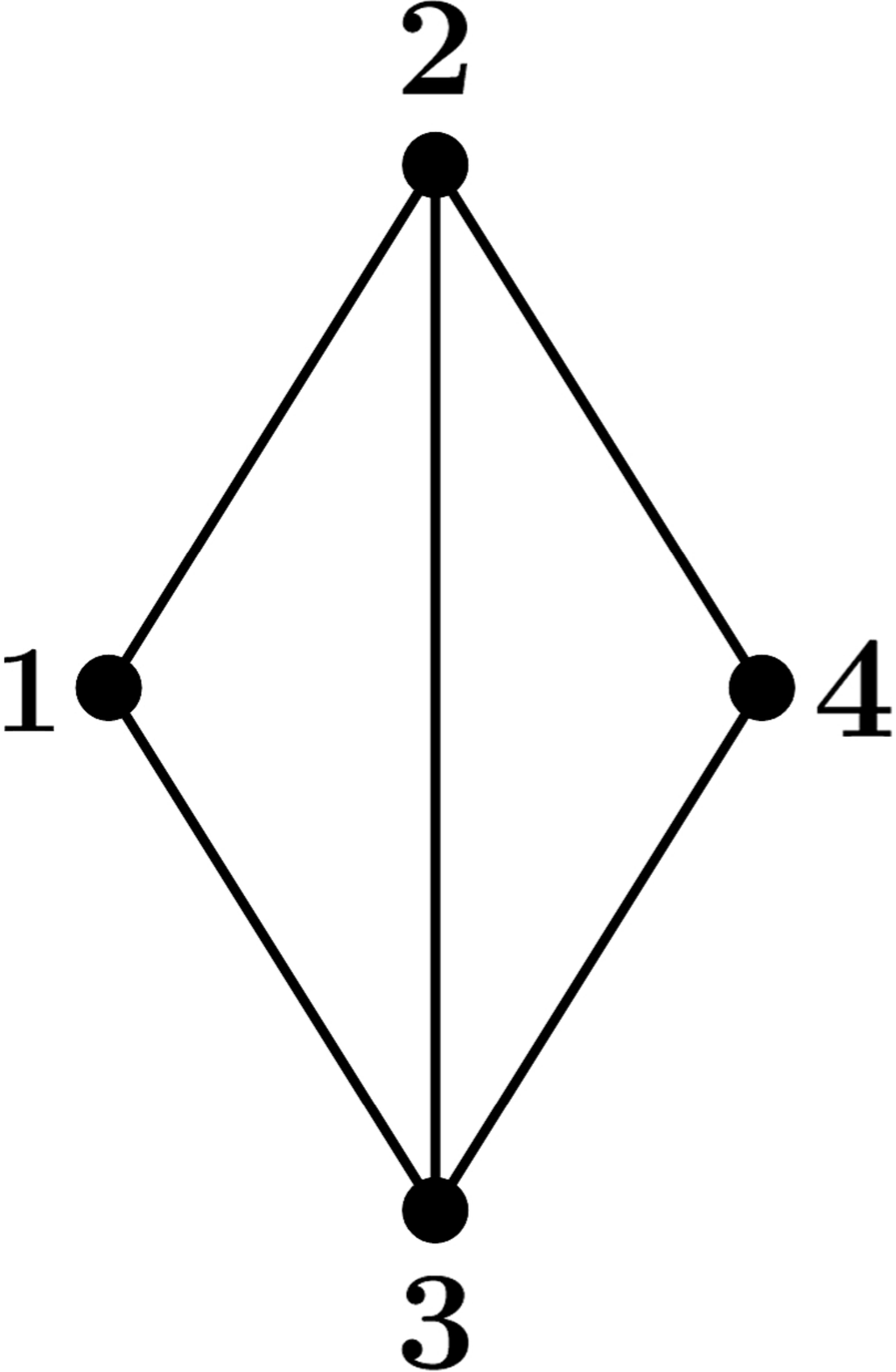}
   \label{fig1:diamond}
 }
 
\label{fig1:graphs}
\caption{Generic Graph Schematics}
\end{figure}

\subsection{Parameter Restrictions}
We now compare the irrepresentable and restricted eigenvalue and smoothness conditions imposed on the model parameters by the different methods.

\subsubsection{Star Graphs}
Consider a star graph $\G(\V,\E)$ with $p$ nodes in Fig~\ref{fig1:star}, where the center node is labeled $1$ and the other nodes are labeled from $2$ to $p$. Following  \cite{RWRY11}, consider the following covariance matrix $\Sigma^*$ parameterized by the correlation parameter $\tau \in [-1,1]$: the diagonal entries are set to $\Sigma^*_{ii} = 1$, for all $i \in V$; the entries corresponding to edges are set to $\Sigma^*_{ij} = \tau$ for $(i,j) \in E$; while the non-edge entries are set as $\Sigma^*_{ij} = \tau^2$ for $(i,j) \notin E$. It is easy to check that $\Sigma^{*}$ induces the desired star graph. With this setup, the irrepresentable condition imposed by the $\ell_1$-regularized Gaussian MLE \cite{RWRY11} entails that $|\tau|(|\tau|+2)<1$ or equivalently $\tau\in(-0.4142,0.4142)$ to guarantee sparsistency. However, our greedy algorithms allow for $\tau \in (-1,1)$ (since $C_{\min}=1-\tau^2$). Under the same setup, the linear neighborhood selection with $\ell_1$-regularization \cite{MeinsBuhl2006} requires $\tau \in (-1,1)$ to guarantee the success.

\subsubsection{Chain Graphs}
Consider a chain (line) graph $\G(\V,\E)$ on $p$ nodes as shown in Fig~\ref{fig1:line}. Again, consider a population covariance matrix $\Sigma^*$ parameterized by the correlation parameter $\tau \in [-1,1]$: set $\Sigma^*_{ij}=\tau^{|i-j|}$. Thus, this matrix assumes a correlation factor of $\tau^k$ between two nodes that are $k$ hops away from each other. It is easy to check that $\Sigma^{*}$ induces the desired chain graph. With this setup, the $\ell_1$-regularized Gaussian MLE \cite{RWRY11} requires $|\tau|^{p-2}\left((p-2)|\tau|+p-1\right)<1$. It is hard to evaluate bounds on $\tau$ in general, but for the case $p=4$ we have $\tau\in(-0.6,0.6)$; for the case $p=10$ we have $\tau\in(-0.75,0.75)$ and for the case $p=100$ we have $\tau\in(-0.95,0.95)$. Our greedy algorithms on the other hand allow for $\tau\in(-1,1)$ (since $C_{\min}=(1-\tau^2)f_p(\tau)$ for some function $f_p(\tau)$ that depends on $p$ and satisfies $f_p(\tau)>C_p$ for all $\tau$ and some constant $C_p$ depending only on $p$). Under the same setup, the linear neighborhood selection with $\ell_1$-regualrization \cite{MeinsBuhl2006} only imposes $\tau\in(-1,1)$, similar to our greedy methods.

\subsubsection{Diamond Graph}
Consider the diamond graph $\G(\V,\E)$ on $4$ nodes with the nodes labeled as in Fig~\ref{fig1:diamond}. Given a correlation parameter $\tau\neq 0$, let $\Sigma^*$ be the population covariance matrix with $\Sigma^*_{ii}=1$ and $\Sigma^*_{ij}=\tau$ except $\Sigma^*_{23}=0$ and $\Sigma^*_{14}=2\tau^2$. It is easy to check that $\Sigma^{*\,-1}$ induces the desired graph. With this setup, the $\ell_1$-regularized Gaussian MLE~\cite{RWRY11} requires $4|\tau|\left(|\tau|+1\right)<1$ or equivalently $\tau\in(-0.2017,0.2017)$. Our greedy algorithm allows for $\tau\in(-0.7071,0.7071)$ (since $C_{\min}=1-2\tau^2$). Under the same setup, the linear neighborhood selection with $\ell_1$-regualrization \cite{MeinsBuhl2006} requires $2|\tau|<1$ or equivalently that $\tau\in(-0.5,0.5)$ to guarantee the success. Unlike the previous two examples, this is a strictly stronger condition than that imposed by our greedy methods.

\section{Experimental Analysis}
\begin{figure}[t]
	\renewcommand{\figurename}{Fig}
	\centering
	\subfigure[Chain (Line Graph)]{\includegraphics[width=0.48\textwidth]{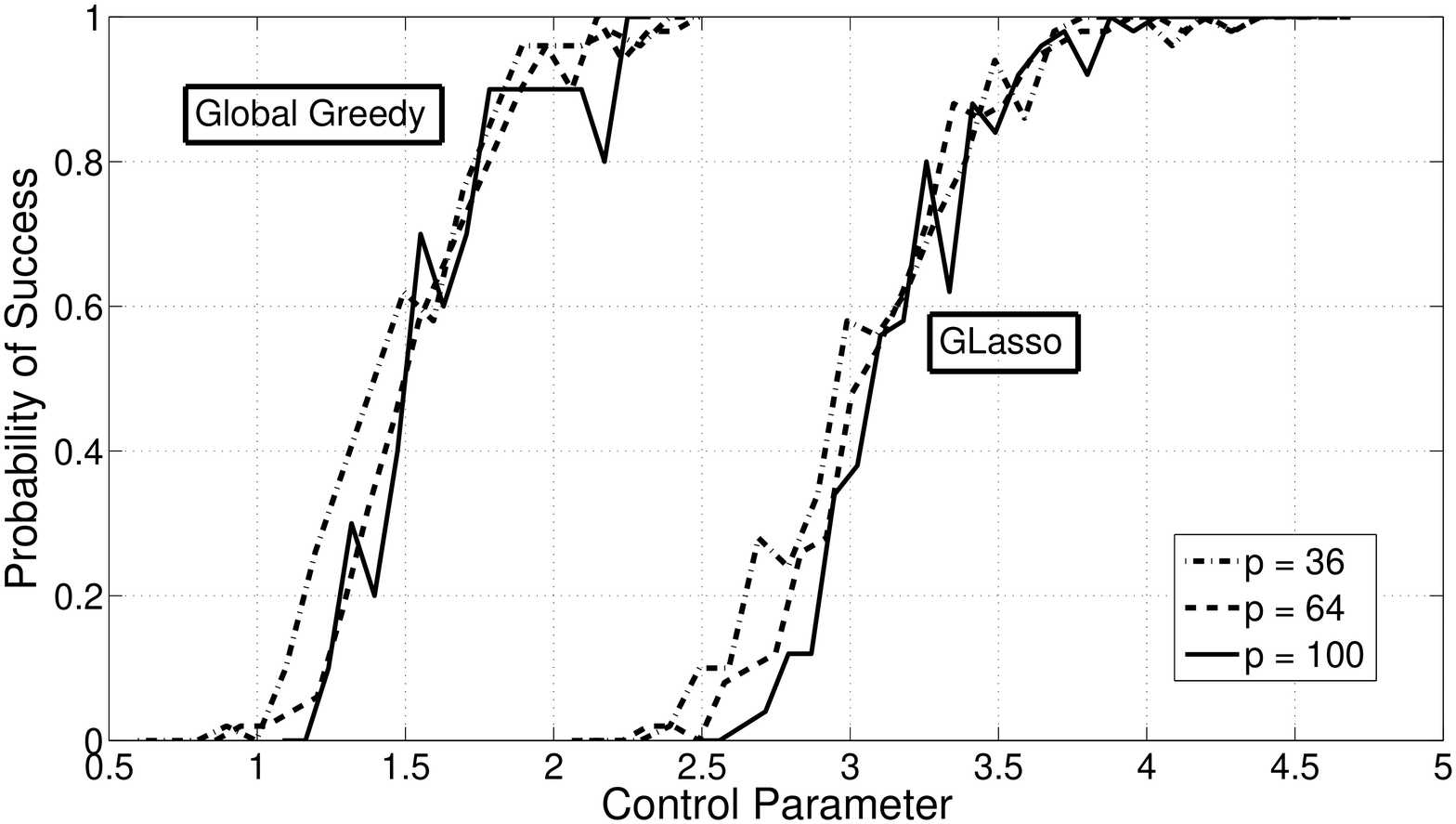}}
	\subfigure[4-Nearest Neighbor (Grid Graph)]{\includegraphics[width=0.48\textwidth]{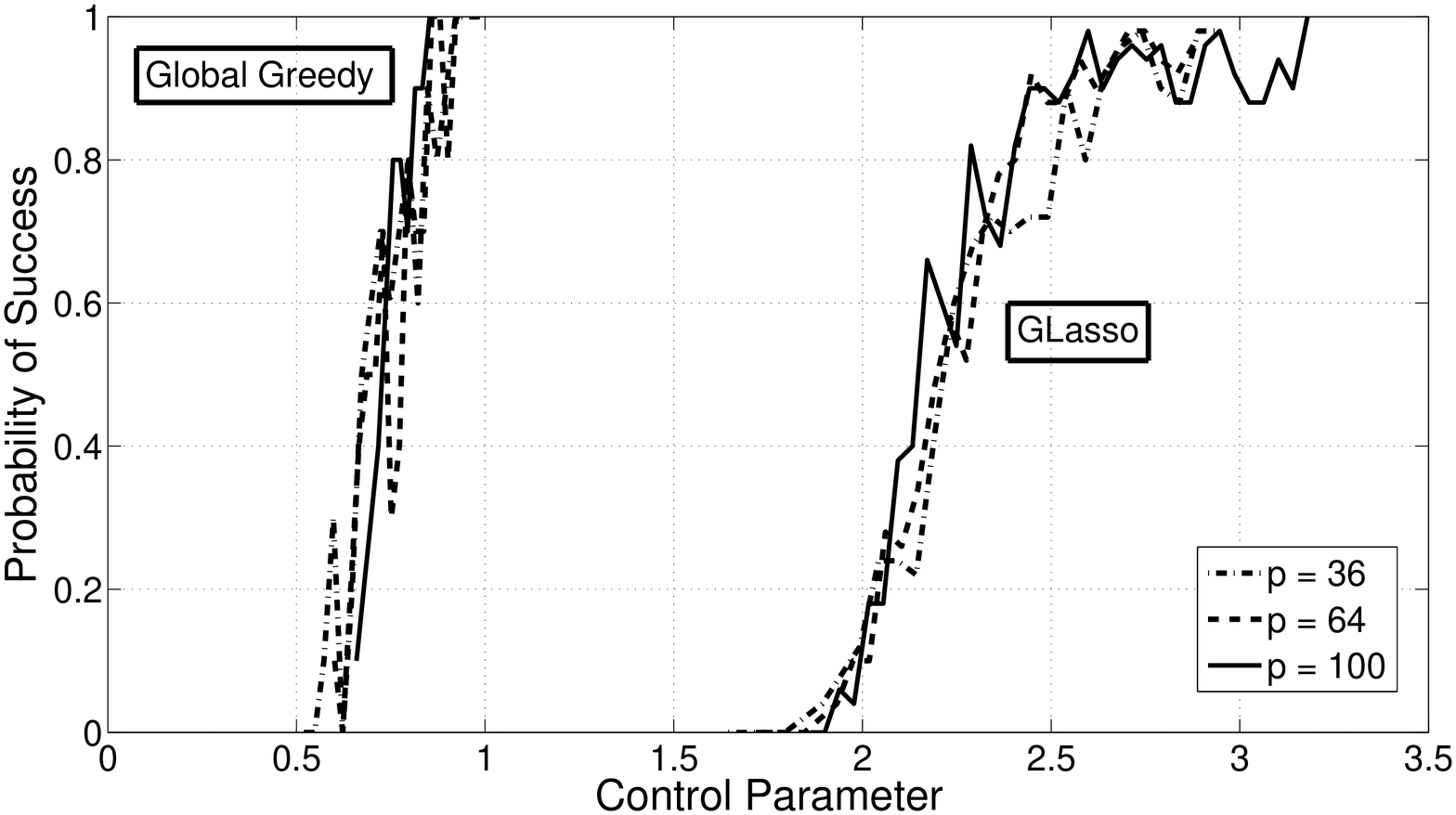}}
	\caption{Plots of success probability $\mathbb{P}[\suppHat = \suppStar]$ versus the control parameter $\beta(n,p,d)=n/[70d\log(p)]$ for (a) chain $(d=2)$ and (b) 4-nearest neighbor grid $(d=4)$ using both Algorithm \ref{Alg:globalgreedyalg} and $\ell_1$-regularized Gaussian MLE (Graphical Lasso).  As our theorem suggests and these figures show, the Global Greedy algorithm requires less samples to recover the exact structure of the graphical model.}
	\label{fig:globalplots}
\end{figure}  

\begin{figure}[t]
	\renewcommand{\figurename}{Fig}
	\centering
	\subfigure[Chain (Line Graph)]{\includegraphics[width=0.48\textwidth]{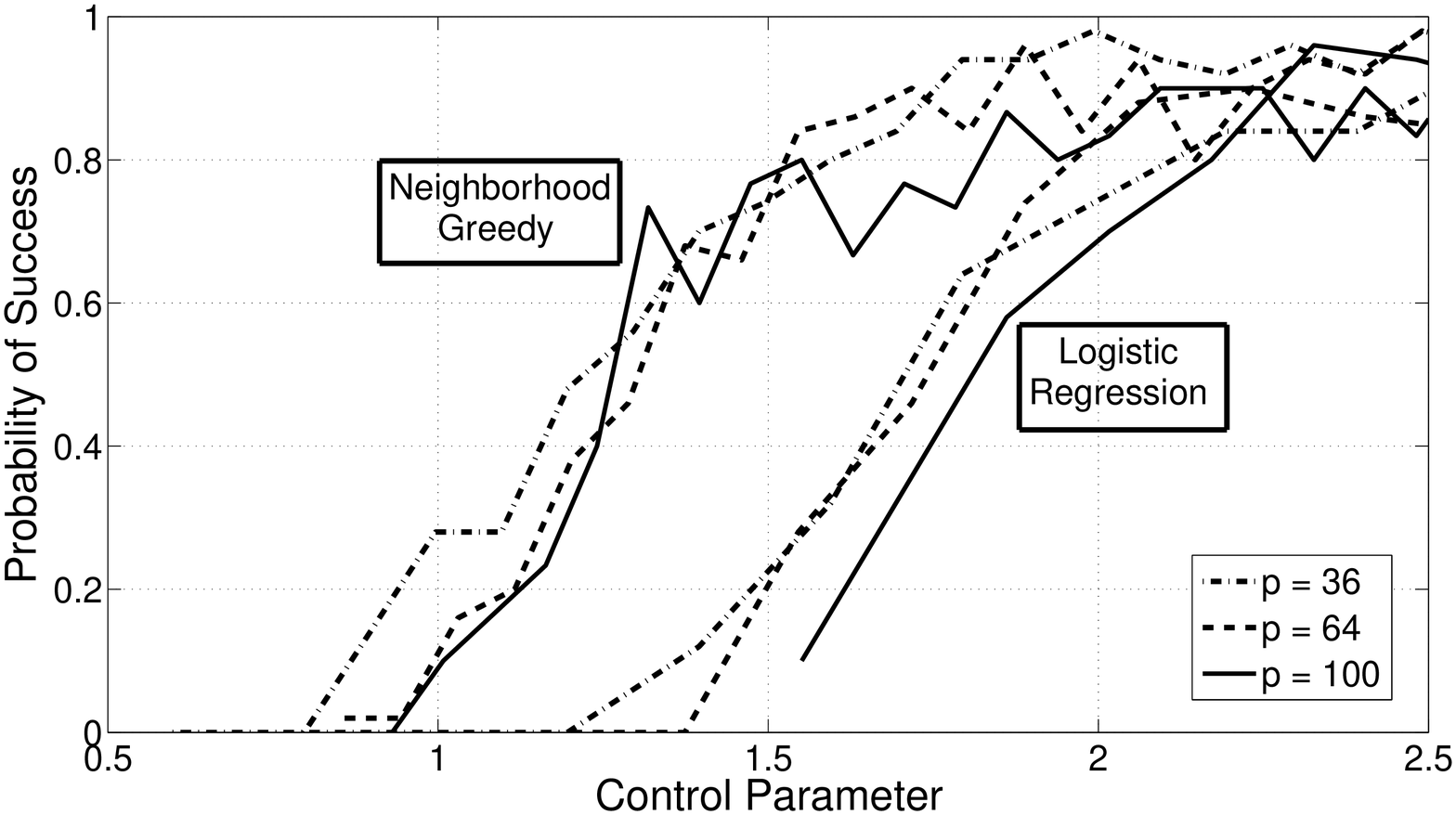}}
	\subfigure[Star Graph]{\includegraphics[width=0.48\textwidth]{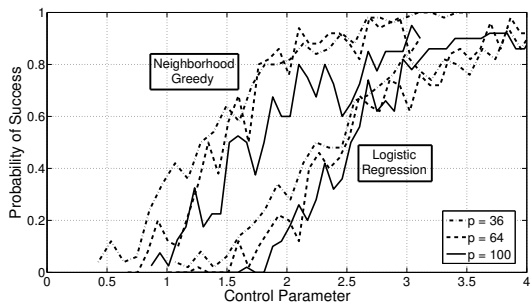}}
	\caption{Plots of success probability $\mathbb{P}[\widehat{\mathcal{N}}_{\pm}(\svert)=\mathcal{N}^*(\svert),\forall\svert\in\vertex]$ versus the control parameter $\beta(n,p,d)=n/[70d\log(p)]$ for (a) chain $(d=2)$ and $\beta(n,p,d)=n/[200\log(dp)]$ for (b) star graph $(d=0.1p)$ using both Algorithm \ref{Alg:nbdgreedyalg} and nodewise $\ell_1$-regularized linear regression (Neighborhood Lasso).  As our theorem suggests and these figures show, the Neighborhood Greedy algorithm requires less samples to recover the exact structure of the graphical model.}
	\label{fig:neighborhoodplots}
\end{figure}

In this section we will outline our experimental results in testing the effectiveness of both Algorithms \ref{Alg:globalgreedyalg} and \ref{Alg:nbdgreedyalg} in a simulated environment.  

\subsection{Optimization Method}
Our greedy algorithm consists of a single variable optimization step where we try to pick the best coordinate. This step can be run in parallel for all single variables to achieve maximum speedup. For greedy neighborhood selection, the single variable optimization is a relatively simple operation, however for the global model selection algorithm (log-det optimization), we would like to provide a fast single variable optimization method to avoid a continual log-det calculation. Following the result in \cite{SchRis}, we have

\vspace{-0.4cm}
\footnotesize\begin{equation}
\begin{aligned}
&\det\left(\widehat{\Theta}^{(k-1)}+\alpha(e_{ij}+e_{ji})\right) =\det\left(\widehat{\Theta}^{(k-1)}\right)\\ &\qquad\qquad\left((1+\alpha(\widehat{\Theta}^{(k-1)})^{-1}_{i,j})^2 - \alpha^2(\widehat{\Theta}^{(k-1)})^{-1}_{ii}(\widehat{\Theta}^{(k-1)})^{-1}_{jj}\right)
\end{aligned}
\nonumber
\end{equation}\normalsize

This entails that 
\footnotesize\begin{equation}
\begin{aligned}
\alpha^* &= \arg\min_{\alpha}\; \tracer{\widehat{\Theta}^{(k-1)}+\alpha(e_{ij}+e_{ji})}{\SigHat^\numobs}\\ &\qquad\qquad\qquad\qquad- \log \det(\widehat{\Theta}^{(k-1)}+\alpha(e_{ij}+e_{ji}))\\
&=\frac{\SigHat^\numobs_{ij}-(\widehat{\Theta}^{(k-1)})^{-1}_{ij}}{(\widehat{\Theta}^{(k-1)})^{-1}_{ii}(\widehat{\Theta}^{(k-1)})^{-1}_{jj} -(\widehat{\Theta}^{(k-1)})^{-1}_{ij}(\widehat{\Theta}^{(k-1)})^{-1}_{ij}}
\end{aligned}
\nonumber
\end{equation}\normalsize
This closed-form solution simplifies the single variable optimization step in our algorithm and avoids continual calculation of $\log\det(\widehat{\Theta})$.

\subsection{Experiments}
To present a formal experimental analysis for both Algorithm \ref{Alg:globalgreedyalg} and Algorithm \ref{Alg:nbdgreedyalg} we simulated zero-mean Gaussian inverse covariance estimation, or GMRF structure learning, for various graph types and scalings of $(n,p,d)$.  For the Global Greedy method we experimented using chain ($d=2$) and grid ($d=4$) graph types with sizes of $p\in\{36,64,100\}$.  For the Neighborhood Greedy method we experimented using chain ($d=2$) and star ($d=0.1p$) graph types with sizes of $p\in\{36,64,100\}$.  Figure 1 outlines the schematic structure for each graph type.  For each algorithm, we measured performance by completely learning the true support set $\suppStar$ pertaining to the non-zero inverse covariates (graph edges).  If $\suppStar$ was completely learned then we called this a \emph{success} and otherwise we called it a \emph{failure}. Using a batch size of $50$ trials for each scaling of $(n,p,d)$ we measured the probability of success as the average success rate.  For both algorithms we used a stopping threshold $\EpsilonStop=\frac{cd\log p}{n}$ where $d$ is the maximum degree of the graph, $p$ is the number of nodes in the graph, $n$ is the number of samples used, and $c$ is a constant tuning parameter, as well as a backwards step threshold of $v=0.5$.  We compared Algorithm \ref{Alg:globalgreedyalg} to that of $\ell_1$-regularized Gaussian MLE (Graphical Lasso) as discussed in \cite{FriedHasTib2007} and \cite{RWRY11} using the \emph{glasso} implementation from Friedman et al.\ \cite{FriedHasTib2007}.  We compared Algorithm \ref{Alg:nbdgreedyalg} to that of neighborhood based $\ell_1$-regularized linear regression (Neighborhood Lasso) using the \emph{glmnet} generalized Lasso implementation, also from Friedman et al.\ \cite{FHT10}.  Both \emph{glasso} and \emph{glmnet} use a regularization parameter $\lambda=c\sqrt{\frac{logp}{n}}$ which was optimally set using $k$-fold cross validation.\\  
\\
Figure \ref{fig:globalplots} plots the probability of successfully learning $\suppStar$ vs the control parameter $\beta(n,p,d)=\frac{n}{70d\log p}$ for varying number of samples $n$ for both Algorithm \ref{Alg:globalgreedyalg} and Graphical Lasso.  Figure \ref{fig:neighborhoodplots} plots the probability of successfully learning $\suppStar$ vs the control parameter $\beta(n,p,d)=\frac{n}{70d\log p}$ for the chain graph type and $\beta(n,p,d)=\frac{n}{200\log (dp)}$ for the star graph type for both Algorithm \ref{Alg:nbdgreedyalg} and neighborhood based $\ell_1$-linear regression.  Both figures illustrate our theoretical results that the Greedy Algorithms require less samples ($O(d\log p)$) than the state of the art Lasso methods ($O(d^2 \log p)$) for complete structure learning.

\bibliographystyle{plainnat}
\bibliography{greedy_gmrf,sml,smlII}

\end{document}